\documentclass{article}

\PassOptionsToPackage{sort, numbers, compress}{natbib}

\usepackage{natbib}

\usepackage[utf8]{inputenc} 
\usepackage[T1]{fontenc}    
\usepackage{hyperref}       
\usepackage{url}            
\usepackage{booktabs}       
\usepackage{amsfonts}       
\usepackage{nicefrac}       
\usepackage{microtype}      
\usepackage{xcolor}         
\usepackage{amsmath}

\DeclareMathOperator*{\argmin}{arg\,min}
\usepackage{amsthm}
\newtheorem{theorem}{Theorem}[section]

\usepackage{pgfplots}
\usepackage[export]{adjustbox}
\usepackage{subcaption}
\usepackage[ruled,vlined]{algorithm2e}
\SetKwInput{KwInput}{Input}
\SetKwInput{KwOutput}{Output}
\usepackage{enumitem}

\usepackage[verbose=true,letterpaper]{geometry}
\AtBeginDocument{
  \newgeometry{
    textheight=9in,
    textwidth=5.5in,
    top=1in,
    headheight=12pt,
    headsep=25pt,
    footskip=30pt
  }
}

\title{DeltaBound Attack: Efficient decision-based attack in low queries regime}

\author{
  Lorenzo Rossi \\
  Politecnico di Milano\\
  \texttt{lorenzo17.rossi@polimi.mail.it} \\
}
\begin{document}
\maketitle
\begin{abstract}
  Deep neural networks and other machine learning systems, despite being extremely powerful and able to make predictions with high accuracy, are vulnerable to adversarial attacks. We proposed the \textsc{DeltaBound} attack: a novel, powerful attack in the hard-label setting with \(\ell_2\) norm bounded perturbations. In this scenario, the attacker has only access to the top-1 predicted label of the model and can be therefore applied to real-world settings such as remote API. This is a complex problem since the attacker has very little information about the model. Consequently, most of the other techniques present in the literature require a massive amount of queries for attacking a single example. Oppositely, this work mainly focuses on the evaluation of attack's power in the low queries regime (\(\leq 1000\) queries) with \(\ell_2\) norm in the hard-label settings. We find that the \textsc{DeltaBound} attack performs as well and sometimes better than current state-of-the-art attacks while remaining competitive across different kinds of models. Moreover, we evaluate our method against not only deep neural networks, but also non-deep learning models, such as Gradient Boosting Decision Trees and Multinomial Naïve Bayes.
  
\end{abstract}

\section{Introduction}
Machine learning systems are not robust against adversarial examples, meticulously crafted perturbation of the inputs that fool the classifier \citep{biggio2013poisoning, goodfellow2015explaining, narodytska2016simple, sooksatra2022evaluation, strisciuglio2020enhanced}. As Deep Neural Networks and other machine learning models become more and more important in everyday life and safety-critical applications such as autonomous navigation, surveillance systems, medical diagnosis, and malware analysis, it is fundamental to understand their robustness in the worst-case scenario. Evaluating the robustness of models is extremely challenging, so it is critical to develop stronger adversarial attacks to better understand the limits and capabilities of machine learning models. Many adversarial attacks have been proposed, however, many of them focus on the white-box setting, where the attacker has full knowledge of the victim model. However, in the real world scenario, generally, the attacker knows neither the kind of models used nor the weights and parameters of the models. To overcome the limitations of the white-box setting various black-box settings were defined. In the black-box settings, generally, the attacker has limited information about the model and the defender architecture, for example, it might not know the weights of the model or it has no access to the gradient of the model. In addition, adversarial attacks can be targeted (in this case the goal is to cause the model to classify the samples to a chosen class) or untargeted (in this case the goal is to cause the model to classify the samples to a different class from the initial ones).

This work presents a black-box hard-label attack with \(\ell_2\) norm bounded perturbations. Particularly, it focuses on the low queries budget, about 1000 queries.
In summary, the contributions of our work are:
\begin{itemize}
    \item We propose \textsc{DeltaBound} a new attack that achieves results comparable or better than the other state-of-the-art attacks in the low query regime.
    \item We integrate and propose various techniques to improve queries efficiency.
    \item We conduct extensive experiments and ablation studies on our method.
\end{itemize}
 
\section{Related works}
We will first introduce different types of threat models and then we analyze the most related attacks in the literature.

\textbf{White-box attack.} In this setting, the attacker has full knowledge about the victim model, in particular on the learned weights, parameters and training data of the model. Many gradient-based algorithms can be found in the literature such as PGD \citep{madry2019deep}, which uses projected gradient descent and  \textsc{Carlini \& Wagner Attack}\citep{carlini2017evaluating}, which performs gradient descent on a well-designed custom function.

\textbf{Black-box attack.} In the black-box setting, the attacker has no direct knowledge of the model and it only knows very limited information. There exist various black-box settings such as transfer-based attack \citep{liu2017delving, papernot2017practical, han2022enhancing}, where the attack has generally access to some information about the structure of the model or the training data, and soft-label \citep{aldujaili2019bit, andriushchenko2020square}, where the attack has access to the class distribution output of the model. This work focuses on an even stricter setting, called hard-label \citep{cheng2020signopt}, where only the final decision is available, i.e. the top-1 predicted class. The first work to analyze the hard-label setting was the \textsc{Boundary Attack}\citep{brendel2018decisionbased}, where they perform rejection sampling along the decision boundary. Other works were proposed exploiting different proprieties of the decision boundary \citep{cheng2018queryefficient, chen2020hopskipjumpattack, rahmati2020geoda}. The vast majority of the methods tries to find better way to estimated the local hyperplane of the decision boundary, and only a few of them tried different approaches \citep{alzantot2019genattack, maho2020surfree}.

\citet{cheng2018queryefficient} proposed \textsc{Opt-attack} based on a novel reformulation of the classical hard-label black-box attack into an optimization problem that finds a direction that minimizes the distance between the sample and the decision boundary.
For a given example \(x_0\), a correct label \(y_0\) and a black-box function \(f \colon {\rm I\!R}^{d} \mapsto \{1,...,K \} \).

\noindent \textbf{Untargeted attack:}
\begin{equation}
 g(\theta) = \argmin_{\lambda > 0} f(x_{0} + \lambda \frac{\theta}{\|\theta\|}) \neq y_{0}
\end{equation}
\textbf{Targeted attack (given target \textit{t}):}
\begin{equation}
 g(\theta) = \argmin_{\lambda > 0} f(x_{0} + \lambda \frac{\theta}{\|\theta\|}) = t
\end{equation}
\begin{equation}
\min_{\theta \in {\rm I\!R}^{d}} g(\theta)
\end{equation}
\(\lambda\) is the distance from \(x_0\) to the decision boundary given the direction \(\theta\). Therefore, the resulting adversarial example is \(x^*_0=x_0+g(\theta)\frac{\theta}{\|\theta\|}\).
\textsc{Opt-attack} solves the optimization problem by using a zeroth-order approximation optimizer. Moreover, they decided to evaluate \(g(\theta)\) using a binary search like algorithm.

An improvement of \textsc{Opt-attack} was proposed by \citet{cheng2020signopt} called \textsc{Sign-opt attack}. They applied a single query oracle to compute \(sign(g(\theta_1)-g(\theta_0))\).
\begin{equation}
  sign(g(\theta_1)-g(\theta_0)) = 
  \begin{cases}
    +1, & \text{If \(f(x_0+g(\theta_0)\frac{\theta_1}{\|\theta_1\|}) = y_0\),} \\
    -1, & \text{Otherwise.}
  \end{cases}
\end{equation}
This single query oracle allows estimating the gradient of \(g(\theta)\) more efficiently.

Another important black-box hard-label attack is \textsc{Surfree} proposed by \citet{maho2020surfree}. This work exploits the geometrical proprieties of the Convolution Neural Networks. In particular, two of the main contributions are the sampling method based on DCT (Discrete Cosine Transformation), which allows making the perturbation less perceptible, and a geometrical mechanism based on the assumption of a hyperplane decision boundary, which is a common assumption in large Deep Neural Networks.

\section{Problem statement}
We introduce the following notations. We consider \(x_0\) the initial sample, \(y_0\) the ground truth label, \(K\) the number of classes, \(d\) is the number of dimensions in the input space and \(f \colon {\rm I\!R}^d \mapsto \{1,...,K \}\) the black-box classifier. Moreover, we consider \(\|x\|\) as the \(\ell_2\) norm of the vector \(x\).
The attacker does not know the function \(f(x)\), but it can observe the function \(f(x)\) a limited number of times, depending on the budget (maximum number of queries). Our goal (in the untargeted case) is to find an adversarial example in \(\ell_2\) norm that it is define as
\begin{equation}
x_0^* = \argmin_{x^* \in {\rm I\!R}^{d}} \|x_0-x^*\| \qquad s.t. \quad f(x^*) \neq y_0.
\end{equation}

In this work, we use the same optimization formulation proposed in \citet{cheng2018queryefficient} (\ref{eq:adv_def_cheng}). In particular, we consider the hard-label black-box attack setting as an optimization problem that finds a direction that minimizes the distance between the sample and the decision boundary, given a limited budget (a maximum number of queries).
Thus, in (\ref{eq:adv_def_cheng}) we optimize \(\theta\), which is the current perturbation direction.

\begin{equation}
\label{eq:adv_def_cheng}
\min_{\theta \in {\rm I\!R}^{d}} g(\theta) \quad g(\theta) = \argmin_{\lambda > 0} f(x_{0} + \lambda \frac{\theta}{\|\theta\|}) \neq y_0
\end{equation}
The resulting adversarial example from the initial problem formulation is \(x_0^*=x_0+g(\theta)\frac{\theta}{\|\theta\|}\).

\section{Proposed method}
Our method uses an iterative local random search, at every iteration, it tries to slightly modify the current best direction \(\theta_{t-1}\) and checks whether the function \(g\) is smaller by at least a factor \(\delta_{t}\) compared with the current best direction (therefore \(\delta\) is the minimum factor of improvement). The main idea behind this approach is that an exact evaluation of \(g(\theta)\) is very expensive (it generally requires between 4 and 8 queries to obtain a good approximation), while a comparison, which means asking whether \(g(\theta)\) smaller or greater than a specific value, can be done with just one query.

In particular, the method works in the following way, at every iteration, as long as it has not exceeded the maximum number of queries available, the attacker tries to add a small perturbation \(u\) to the current best direction \(\theta_{t-1}\). We call \(v\) the new candidate direction, that is generated by first summing \(\theta_{t-1}\) and \(u\), a randomly generated vector. Then, a Gram-Schmidt procedure is applied to make it orthogonal with the previous directions \(\{\theta_k\}^{t-2}_{k=0}\).
We accept a candidate vector \(v\) only if \(g(v) \leq g(\theta_{t-1}) \delta_{t}\). Therefore, \(\delta_{t}\) can be defined as the factor of improvement, which means that we accept a new vector, only if it is better by at least a factor \(\delta_{t}\). 
In order to check whether \(g(v) \leq g(\theta_{t-1}) \delta_{t}\), we used a variation of the single query oracle proposed in \citet{cheng2020signopt}. Particularly, \(sign(g(v)-g(\theta_{t-1}) \delta_{t}) = -1\) is equivalent to \(g(v) \leq g(\theta_{t-1}) \delta_{t}\). Thus, we update the direction \(\theta_{t-1}\) only if 
\begin{equation}
    f(x_0+g(\theta_{t-1}) \delta_t \frac{v}{\|v\|}) \neq y_0.
\end{equation}
Algorithm \ref{algo:iter} shows one iteration of the \textsc{DeltaBound} attack. The first direction \(\theta_0\) is randomly generated, and therefore the the first \(g_{best}\) is equal to \(g(\theta_0)\). Additionally, a 2D example of the algorithm can be seen in \ref{appendix:D}.

\begin{algorithm}[H]
\label{algo:iter}
\DontPrintSemicolon
\KwInput{best distance \(g_{best}\), best direction \(\theta_{t-1}\), current iteration \(t\)}
\KwOutput{\(\theta_{t}\)}

\(u \gets\) Sample a random vector \;
\(v \gets  \textbf{proj}_{\textbf{span}(\{\theta_k\}^{t-2}_{k=0})} \perp (\theta_{t-1} + u)\) \tcp{Gram-Schmidt orthogonalization}\;

\If(\tcp*[h]{equivalent to \(g(v) \leq g_{best} \delta_{t}\)}){\(f(x_0+g_{best} \delta_t \frac{v}{\|v\|}) \neq y_0\)}{ 
    \(\theta_{t} \gets v\) \;
    \(g_{best} \gets g(\theta_{t})\) \;
}\Else{
    \(\theta_{t} \gets \theta_{t-1}\) \;
}
\KwRet \(\theta_t\)
\caption{One iteration of the \textsc{DeltaBound} Attack.}
\end{algorithm}

\subsection{How to evaluate \(g(\theta_{t})\)}
We evaluate \(g(\theta_{t})\) by using a binary search like algorithm, however compared to previous works \citep{cheng2018queryefficient, cheng2020signopt} in our attack we have an upper bound on \(g(\theta_{t})\), which is \(g(\theta_t) \leq g(\theta_{t-1}) \delta_{t}\), therefore we do not need a fine-grained search and hence reducing the number of queries required to evaluate \(g(\theta)\).

Moreover, the number of queries for a single function evaluation is \(\mathcal{O}(\log{{\epsilon}^{-1} g(\theta_{t-1}))}\), where \(\epsilon\) is the wanted precision. 
If we set the precision to \((1-\delta_{t}) g(\theta_{t-1})\), we not only have the guarantees to obtain a smaller value at the end of the binary search, but also we achieve \(\mathcal{O}(\log{(1-\delta_{t})^{-1}))}\) queries per function evaluation.

It is possible to further reduce the number of queries by noticing that \(\frac{g(\theta_{t})}{g(\theta_{t-1})}\) is generally close to 1 (see Appendix \ref{appendix:A}). So we can reduce the number of queries by estimating the distribution of \(\frac{g(\theta_{t})}{g(\theta_{t-1})}\) (see Appendix \ref{appendix:C}). We estimate the distribution of \(\frac{g(\theta_{t})}{g(\theta_{t-1})}\) as an empirical distribution function by keeping a list of previous ratios. In particular, for our method we need \(G(p)\) the inverse cumulative distribution function (ICDF) of the empirical distribution of \(\frac{g(\theta_{t})}{g(\theta_{t-1})}\).

The modified algorithm (see Algorithm \ref{algo:bin_search}) is quite similar to the classical binary search, but first of all it bisects based on the \(  f(x_0 + g(\theta_{t-1}) G(mid) \frac{\theta_{t}}{\|\theta_{t}\|})  \neq y_0 \), which is equivalent to \(g(\theta_t) \leq g(\theta_{t-1}) G(mid)\) and it does not test the middle point in the distance space, but it finds the point to test by using the ICDF of the empirical distribution function. The result of Algorithm \ref{algo:bin_search} is the value of \(g(\theta_t\).

\begin{algorithm}[H]
\label{algo:bin_search}
\DontPrintSemicolon
\KwInput{\(x_0, y_0\) The original value and label, \(\theta_{t}\) Perturbation, \(\epsilon\) Precision}
\KwData{\(G(p)\) The estimated ICDF of \(\frac{g(\theta_{t})}{g(\theta_{t-1})}\), \(f(x)\) Classifier}

\(l \gets 0\)\;
\(r \gets 1\)\;
\While{\(G(r) - G(l)  > 1 - \epsilon\)}{
    \(mid \gets \frac{l+r}{2}\) \;
    \If{\(  f(x_0 + g(\theta_{t-1}) G(mid) \frac{\theta_{t}}{\|\theta_{t}\|})  \neq y_0 \)}{  
        \(l \gets mid\)\;
    }\Else{
        \(r \gets mid\)\;
    }
}
\KwRet \(g(\theta_{t-1})G(r)\)
\caption{Improved Binary Search.}
\end{algorithm}

\subsection{How to choose \(u\)}
In this section, we discuss how to sample and scale \(u\) the randomly generated vector used in Algorithm \ref{algo:iter}.
To simply our analysis, we consider \(u\) in the form of \(u := z p(t)\), where \(z\) is a vector sampled from some distribution and \(p(t)\) is a function \(p \colon {\rm I\!N} \mapsto {\rm I\!R^+} \), which depends on the current iteration \(t\).

\subsubsection{Sampling distribution of \textit{z}}
\(z\) is the nonscaled random vector of a candidate modification in the perturbation direction \(\theta\). The sampling strategy for \(z\) is crucial for the efficiency of the attack, even more, if the number of queries is two orders of magnitude smaller than the image space, such as in the \textsc{ImageNet} case where the image space has \(3 \times 224 \times 224=150528\) dimensions compared to 1000 queries available.

As suggested by recent works\citep{maho2020surfree, andriushchenko2020square, li2020qeba}, the normal distribution is not the most suited kind of perturbation in the image settings for attacking current computer vision models, such as ResNet \citep{he2015deep}. We evaluate two kinds of techniques, the first one is sampling \(z\) from a normal distribution and the second one exploits some image proprieties to make the attack less perceptible. The latter is a method similar to the one proposed by \citet{maho2020surfree} and it is a random process in which samples \(z\) from a particular distribution according to Algorithm \label{algo:dct}. Firstly, it creates a 0/1 filtering mask and only keeping the \(\rho\)\% of the lowest frequency subband (in order to create a low pass filter), then it is multiplied by a random vector \(r\) where each element is sampled uniformly over \(\{-1, 0, 1\}\). Finally, the resulting vector is converted back to pixels space using an Inverse Discrete Cosine Transformation (\(DCT^{-1}\)). Figure \ref{fig:band} shows some examples of perturbation directions for different values of \(\rho\).

\begin{algorithm}[H]
\DontPrintSemicolon
\(U \gets\) 0/1 filtering mask of the \(\rho \%\) of the lowest frequency subband\;
\(d \gets \mathcal{U}_{\{-1, 0, 1\}} \) \tcp{Same shape of U and independent for each dimension} \;
\(u \gets DCT^{-1}(d \cdot U)\) \;
\KwRet \(u\)
\caption{DCT sampling strategy.}
\label{algo:dct}
\end{algorithm}

\begin{figure}[h]
    \centering
    \includegraphics[width=0.25\textwidth]{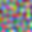}
    \includegraphics[width=0.25\textwidth]{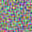}
    \includegraphics[width=0.25\textwidth]{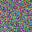}
    \caption{Various perturbation directions using DCT sampling strategy and different low pass filtering rate. From left to right: \(\rho=\{0.1, 0. 5, 0.9\}.\)}
    \label{fig:band}

\end{figure}

\subsubsection{How to choose \(p(t)\)}
The scale factor \(p(t)\) is also an important parameter. This parameter can be seen as a trade-off between exploitation (local search with smaller changes of the current direction \(\theta_t\)) and exploration (searching with larger changes of the current direction \(\theta_t\)).

\begin{theorem}[Probability of updates]
Let \(\Phi(x)\) be the cumulative distribution function of the standard normal distribution.
If \(g(x)\) is differentiable, the vector \(u \sim \mathcal{N}(0, \mathcal{I} p(t)^{2})\) and \(\|\nabla g(\theta_{t})\| \neq 0\), then
\begin{equation}
P[g(\theta_{t}) \leq g(\theta_{t} + u) \delta_{t}] = 1 - \Phi(\frac{g(\theta_{t})(1-\delta_{t}) + o(\|u\|)}{\|\nabla g(\theta_{t})\|p(t)})
\end{equation}
\label{the:prob_up}
\end{theorem}
\begin{proof}
See Appendix \ref{theo:A.1}
\end{proof}

By Theorem \ref{the:prob_up}, in order to keep the probability of finding a better value for \(g(\theta_{t})\) greater than a constant, we must ensure that \(\frac{g(\theta_{t})(1-\delta_{t}) + o(\|u\|)}{\|\nabla g(\theta_{t})\|p(t)} \leq C \) for a constant \(C \in {\rm I\!R}\). Therefore, we have to choose \(\delta_t\) and \(p(t)\) (\(\|u\|\) is strongly correlated with \(p(t)\)) such that \(\frac{g(\theta_{t})(1-\delta_{t}) + o(\|u\|)}{\|\nabla g(\theta_{t})\|p(t)}\) has an upper bound. First of all, we know that if the algorithm converges to a local minimum, then \(\|\nabla g(\theta_{t})\|\to 0\). Thus, both \(p(t) \to 0 \) and \(\delta_{t} \to 1\) must hold, as \(t \to \infty\), in order to guarantee the probability of updates lower bounded.
In practice, when we analyze the problem in the low queries regime it is not necessary to have \( p(t) \to 0 \) (see Appendix \ref{appendix:B}).

We tried various schedule functions for \(p(t)\).
\begin{itemize}
  \item[]  \textbf{const}  \( p(t) = 1 \)
  \item[]  \textbf{linear} \( p(t) = \frac{1}{t + 1} \)
  \item[]  \textbf{sqrt}   \( p(t) =  \frac{1}{\sqrt{t + 1}} \)
  \item[]  \textbf{log}    \( p(t) = \frac{1}{\log{t + 2}} \)
\end{itemize}

\subsection{How to choose \(\delta_{t}\)}
First of all, we must choose \(\delta_t\) such that \(0 < \delta_{t} < 1 \), as \(\delta_t\) represents the factor of improvement required.
Moreover, if the algorithm finds a local minimum then \(\|\nabla g(\theta_{t})\|\to 0\). Therefore, as describe in subsection 4.2.2, we must design \(\delta_{t}\) such that \(\delta_{t} \to 1\).

\begin{theorem}[Gradient Bound] 
Let \(g(x)\) be differentiable, \(\delta_{t} = 1 - \|\nabla g(\theta_{t})\| \) and \(g(\theta_{t+1}) \leq \delta_{t} g(\theta_{t})\).
\begin{equation}
  \|\nabla g(\theta_{t})\| \leq o(\frac{1}{\cos{\angle \big(\nabla g(\theta_{t}), \theta_{t+1}-\theta_{t} \big)}} + \frac{\|\theta_{t+1}-\theta_{t}\|}{g(\theta_{t})})
\end{equation}
\label{the:grad}
\end{theorem}
\begin{proof}
See Appendix \ref{theo:A.2}
\end{proof}

Moreover, Theorem \ref{the:grad} gives us a bound on the gradient of \(g(\theta_{t})\), however it is still not obvious how to choose \(\delta_{t}\).
Because, \(\|\nabla g(\theta_{t})\|\) is not trivial to estimate, various previous works tried to analyze it \citep{fawzi2016robustness, Fawzi_2018_CVPR}, however, we are not aware of any common proprieties across different types of machine learning models.

In our work, we test several approaches, each of them mainly differs from one another on how fast it converges to 1.
Let \(C\in {\rm I\!R}\) be the factor of \(\delta_t\) and let the functions below be possible schedule rate of \(\delta_t\).
\begin{itemize}
  \item[]  \textbf{linear} \( \delta_{t} = 1 - \frac{C}{t+1} \)
  \item[]  \textbf{sqrt}   \( \delta_{t} = 1 - \frac{C}{\sqrt{t+1}} \)
  \item[]  \textbf{log}    \( \delta_{t} = 1 - \frac{C}{\log{t+2}} \)
\end{itemize}

\section{Experimental work and results}
Firstly, we specify the experimental setup and parameters, then we perform an ablation study on the various components, lastly, we compare our attack to recent state-of-the-art hard-label black-box attacks in \(\ell_2\).
\subsection{Experimental setup}
\textbf{Datasets and models}
We use three image datasets (\textsc{MNIST}, \textsc{CIFAR10}, \textsc{ImageNet}), we choose these because they are the most used datasets for adversarial attack benchmarks and one tabular dataset (Breast Cancer Wisconsin dataset\citep{Dua:2019}), we choose this dataset because it is simple (all tested models achieve high accuracy on the test set) and it is a classification task.

For \textsc{MNIST}, we use a pre-trained multilayer perceptron with 2 fully connected layers with 256 neurons with RELU  \footnote{Model weights: \url{http://ml.cs.tsinghua.edu.cn/~chenxi/pytorch-models/mnist-b07bb66b.pth}}. We randomly choose a subset of 150 correctly classified images for the evaluation.

\textsc{CIFAR10} dataset is tackled by using WideResNet-28-10 architectures both for standard and robust ResNet model for Robust model. We employ the pre-trained model available in the RobustBench library\citep{croce2020robustbench}. For the Robust model, we use \citet{rebuffi2021fixing}, because it shows high robust accuracy performance. We randomly select a subset of 100 correctly classified images for the evaluation.

For \textsc{ImageNet}, we use a pre-trained ResNet18, made available for the PyTorch environment. We randomly selected 100 correctly classified images with size \(3 \times 224 \times 224\).

The Breast Cancer Wisconsin dataset is a classical binary classification dataset. It has a total of 569 samples, each sample has 30 dimensions. We evaluate different kinds of models from the sklearn library\citep{scikit-learn}. In particular, we use the following classes \textsc{GradientBoostingClassifier}, \textsc{LogisticRegression}, \textsc{RandomForestClassifier}, \textsc{MultinomialNB}, \textsc{AdaBoostClassifier}, \textsc{DecisionTreeClassifier}, in order to give a wide range of different kind of models. 
We randomly sample 86 data points from the dataset.

\textbf{Setup and Code}
Unless differently stated we use the following parameters \(p(t)=1\), \(\delta_t=\frac{0.05}{\log{t+2}}\). For the sampling strategy, we use DCT with \(\rho=0.1\) in the \textsc{MNIST}, Robust \textsc{CIFAR10} and \textsc{ImageNet} cases, while \(\rho=0.9\) for the Standard \textsc{CIFAR10} case. In the Breast Cancer Wisconsin dataset it is not possible to use the DCT sampling strategy, therefore, we used the normal sampling strategy. We develop \textsc{Deltabound} on top of the FoolBox library \citep{rauber2017foolboxnative}.

\subsection{Evaluation metrics}
The core evaluation metrics is the average \(\ell^2\) distance between the initial image \(x_0\) and the adversarial example \(x^*_0\) (equation 10). For the image datasets, we give a budget of 1000 queries, while we give a budget of 500 queries for the Breast Cancer Wisconsin dataset due to its smaller input space.
\begin{equation}
 \textit{avg}_{\ell^2}  := \frac{1}{n} \sum_{i=1}^{n} \|x_0-x^*_0\|_2
\end{equation}

\subsection{Ablation study}
In this subsection, we evaluate the various components, and we analyze how important each component is.
We tested the various hyper-parameters on \textsc{CIFAR10} with the standard model and the results of this analysis are in Table \ref{tab:avg_l2}.
Therefore, we show that choosing a good \(p(t)\) and \(\delta_t\) schedule and factor are extremely important, and different choices of those parameters could significantly change the performance and effectiveness of \textsc{DeltaBound}.

\begin{table}
\centering
\begin{tabular}{|c|cc|c||c|} 
 \hline
 \(p(t)\) & \(\delta_t\) schedule & \(\delta_t\) factor & \(u\) distribution & average \(\ell_2\) \\
 \hline
 \hline
 linear & log   & 0.01  & normal & 2.04 \\
 sqrt   & log   & 0.01  & normal & 1.06 \\
 log    & log   & 0.01  & normal & 0.87 \\
 const  & log   & 0.01  & normal & 0.79 \\
 \hline
 const & linear   & 0.01 & normal & 1.09 \\
 const & sqrt     & 0.01 & normal & 0.87 \\
 \hline
 const & log      & 0.01 & DCT    & 0.79 \\
 const & log      & 0.05 & DCT    & 0.72 \\
 const & log      & 0.1  & DCT    & 0.88 \\
 \hline
 \end{tabular}
 \caption{Average \(\ell_2\) using the standard model on \textsc{CIFAR10} dataset with a budget of 1000 queries.}
 \label{tab:avg_l2}
\end{table}
\begin{figure}[hhhh]
    \centering
    \begin{subfigure}{0.15\textwidth}
    \includegraphics[width=\textwidth]{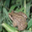}
    \caption{frog}
    \end{subfigure}
    \begin{subfigure}{0.15\textwidth}
    \includegraphics[width=\textwidth]{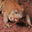}
    \caption{frog}
    \end{subfigure}
    \begin{subfigure}{0.15\textwidth}
    \includegraphics[width=\textwidth]{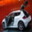}
    \caption{automobile}
    \end{subfigure}
    \begin{subfigure}{0.15\textwidth}
    \includegraphics[width=\textwidth]{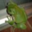}
    \caption{frog}
    \end{subfigure}
    
    \begin{subfigure}{0.15\textwidth}
    \includegraphics[width=\textwidth]{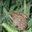}
    \caption{bird}
    \end{subfigure}
    \begin{subfigure}{0.15\textwidth}
    \includegraphics[width=\textwidth]{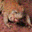}
    \caption{bird}
    \end{subfigure}
    \begin{subfigure}{0.15\textwidth}
    \includegraphics[width=\textwidth]{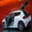}
    \caption{dog}
    \end{subfigure}
    \begin{subfigure}{0.15\textwidth}
    \includegraphics[width=\textwidth]{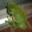}
    \caption{dog}
    \end{subfigure}
    
    \caption{Initial image (top row) and Adversarial example (bottom row) on \textsc{CIFAR10} with the standard model.}
    \label{fig:samples}
\end{figure}

\subsection{Benchmark}
\textbf{Image datasets} In our benchmark, we compare the performance of our method with several state-of-the-art attacks: \textsc{HSJA}\citep{chen2020hopskipjumpattack}, \textsc{Boundary Attack}\citep{brendel2018decisionbased} and \textsc{SurFree} \citep{maho2020surfree}. Moreover, despite not being part of our setting, we also compare with \textsc{Carlini \& Wagner Attack}\citep{carlini2017evaluating}, which is a white-box attack.
For \textsc{HSJA}\citep{chen2020hopskipjumpattack}, \textsc{Boundary Attack}\citep{brendel2018decisionbased} and \textsc{Carlini \& Wagner Attack}\citep{carlini2017evaluating} we used the implementation directly available in the FoolBox library \citep{rauber2017foolboxnative}. While for \textsc{SurFree} \citep{maho2020surfree}, we use the official GitHub implementation \footnote{\url{https://github.com/t-maho/SurFree}}.
Table \ref{tab:comparison} shows a comparison of the different methods across various models and datasets. Our work remains competitive across different kinds of computer vision models and datasets while being better on \textsc{MNIST} and \textsc{CIFAR10} with a robust model.
Figure \ref{fig:samples} displays some attacked images with the \textsc{DeltaBound} attack on \textsc{CIFAR10} with the standard model.

\begin{table} 
\centering
\begin{tabular}{|c||c|c|c|c|} 
 \hline
 Attack &  \textsc{MNIST} & Standard - \textsc{CIFAR10} & Robust - \textsc{CIFAR10} & \textsc{ImageNet}\\ [0.5ex]
 \hline
 \hline
 Our method & \textbf{1.19} & \underline{0.72}  & \textbf{2.84} & \underline{8.15} \\
 \hline
 \textsc{HSJA} & \underline{1.30} & \underline{0.72} & 3.77 & 10.70 \\
 \hline
 \textsc{BoundaryAttack} & 4.95 & 4.55 & 8.97 & 34.22 \\
 \hline
 \textsc{SurFree} & 1.63 & \textbf{0.68} & \underline{3.71} & \textbf{6.73} \\
 \hline
 \hline
 \textsc{C\&W} (white-box) & 0.95 & 0.11 & 0.88 & 0.45 \\
 \hline
 \end{tabular}
\caption{Average \(\ell_2\) with a budget of 1000 queries. \textbf{bold} the best. \underline{underlined} the second best.}
 \label{tab:comparison}
\end{table}

\textbf{Tabular dataset} In Table \ref{tab:tabular}, we show our method compared with \textsc{HSJA}\citep{chen2020hopskipjumpattack} on the breast cancer Wisconsin dataset\citep{Dua:2019}. Our method is consistently better than \textsc{HSJA}, across a wide and diverse range of machine learning methods. We exclude \textsc{Boundary Attack}\citep{brendel2018decisionbased} because it has been proven to be less efficient than \textsc{HSJA}\citep{chen2020hopskipjumpattack}. We also do not consider \textsc{Carlini \& Wagner Attack}\citep{carlini2017evaluating} as it is a white-box attack and it requires access to the model gradient, which many models we analyze do not have. Moreover, we cannot evaluate \textsc{SurFree} \citep{maho2020surfree} as their work focuses on the images settings and particularly their sampling strategy it is not trivially adaptable to the tabular settings.

\begin{table} 
\centering
 \begin{tabular}{|c||c|c|c|c|}
 \hline
 Model & const - log 0.1 (ours) & const - sqrt 0.1 (ours) & HSJA \\ [0.5ex]
 \hline
 \hline
 Gradient Boosting Tree & 0.58 & \textbf{0.51} & 0.67 \\
 \hline
 Logistic Regression & \textbf{5.86} & 5.88 & 6.15 \\
 \hline
 Random Forest & \textbf{0.19} & 0.28 & 0.62 \\
 \hline
 Multinomial Naive Bayes & \textbf{16.45} & 16.49 & 17.22\\
 \hline
 AdaBoost & \textbf{2.41} & 2.54 & 2.79\\ 
 \hline
 Decision Tree & \textbf{0.52} & \textbf{0.52} & 0.92\\
 \hline
 \hline
\end{tabular}
 \caption{Average \(\ell_2\) with a budget of 500 queries using the breast cancer Wisconsin dataset\citep{Dua:2019}. const - log 0.1 means that \(p(t)=1\) and \(\delta_t=1-\frac{0.1}{\log{t+2}}\), while for const - sqrt o.1 means that \(p(t)=1\) and \(\delta_t=1-\frac{0.1}{\sqrt{t+2}}\).\textbf{bold} the best.}
 \label{tab:tabular}
\end{table}

\section{Conclusions}
We propose the \textsc{DeltaBound} attack, an adversarial attack in the black-box hard-label setting with a very limited amount of queries (\(\leq 1000\)). Overall, our attack achieves better results than the other considered attacks from literature in at least one of its configurations.

\section*{Acknowledgement}
The author would like to thank Giacomo Boracchi (Politecnico di Milano) and Loris Giulivi (Politecnico di Milano) for the helpful suggestions.

\medskip
\bibliography{main}

\begin{thebibliography}{28}
\providecommand{\natexlab}[1]{#1}
\providecommand{\url}[1]{\texttt{#1}}
\expandafter\ifx\csname urlstyle\endcsname\relax
  \providecommand{\doi}[1]{doi: #1}\else
  \providecommand{\doi}{doi: \begingroup \urlstyle{rm}\Url}\fi

\bibitem[Al-Dujaili and O'Reilly(2019)]{aldujaili2019bit}
A.~Al-Dujaili and U.-M. O'Reilly.
\newblock There are no bit parts for sign bits in black-box attacks, 2019.

\bibitem[Alzantot et~al.(2019)Alzantot, Sharma, Chakraborty, Zhang, Hsieh, and
  Srivastava]{alzantot2019genattack}
M.~Alzantot, Y.~Sharma, S.~Chakraborty, H.~Zhang, C.-J. Hsieh, and
  M.~Srivastava.
\newblock Genattack: Practical black-box attacks with gradient-free
  optimization, 2019.

\bibitem[Andriushchenko et~al.(2020)Andriushchenko, Croce, Flammarion, and
  Hein]{andriushchenko2020square}
M.~Andriushchenko, F.~Croce, N.~Flammarion, and M.~Hein.
\newblock Square attack: a query-efficient black-box adversarial attack via
  random search, 2020.

\bibitem[Biggio et~al.(2013)Biggio, Nelson, and Laskov]{biggio2013poisoning}
B.~Biggio, B.~Nelson, and P.~Laskov.
\newblock Poisoning attacks against support vector machines, 2013.

\bibitem[Brendel et~al.(2018)Brendel, Rauber, and
  Bethge]{brendel2018decisionbased}
W.~Brendel, J.~Rauber, and M.~Bethge.
\newblock Decision-based adversarial attacks: Reliable attacks against
  black-box machine learning models, 2018.

\bibitem[Carlini and Wagner(2017)]{carlini2017evaluating}
N.~Carlini and D.~Wagner.
\newblock Towards evaluating the robustness of neural networks, 2017.

\bibitem[Chen et~al.(2020)Chen, Jordan, and
  Wainwright]{chen2020hopskipjumpattack}
J.~Chen, M.~I. Jordan, and M.~J. Wainwright.
\newblock Hopskipjumpattack: A query-efficient decision-based attack, 2020.

\bibitem[Cheng et~al.(2018)Cheng, Le, Chen, Yi, Zhang, and
  Hsieh]{cheng2018queryefficient}
M.~Cheng, T.~Le, P.-Y. Chen, J.~Yi, H.~Zhang, and C.-J. Hsieh.
\newblock Query-efficient hard-label black-box attack:an optimization-based
  approach, 2018.

\bibitem[Cheng et~al.(2020)Cheng, Singh, Chen, Chen, Liu, and
  Hsieh]{cheng2020signopt}
M.~Cheng, S.~Singh, P.~Chen, P.-Y. Chen, S.~Liu, and C.-J. Hsieh.
\newblock Sign-opt: A query-efficient hard-label adversarial attack, 2020.

\bibitem[Croce et~al.(2020)Croce, Andriushchenko, Sehwag, Flammarion, Chiang,
  Mittal, and Hein]{croce2020robustbench}
F.~Croce, M.~Andriushchenko, V.~Sehwag, N.~Flammarion, M.~Chiang, P.~Mittal,
  and M.~Hein.
\newblock Robustbench: a standardized adversarial robustness benchmark, 2020.

\bibitem[Dua and Graff(2017)]{Dua:2019}
D.~Dua and C.~Graff.
\newblock {UCI} machine learning repository, 2017.
\newblock URL \url{http://archive.ics.uci.edu/ml}.

\bibitem[Fawzi et~al.(2016)Fawzi, Moosavi-Dezfooli, and
  Frossard]{fawzi2016robustness}
A.~Fawzi, S.-M. Moosavi-Dezfooli, and P.~Frossard.
\newblock Robustness of classifiers: from adversarial to random noise, 2016.

\bibitem[Fawzi et~al.(2018)Fawzi, Moosavi-Dezfooli, Frossard, and
  Soatto]{Fawzi_2018_CVPR}
A.~Fawzi, S.-M. Moosavi-Dezfooli, P.~Frossard, and S.~Soatto.
\newblock Empirical study of the topology and geometry of deep networks.
\newblock In \emph{Proceedings of the IEEE Conference on Computer Vision and
  Pattern Recognition (CVPR)}, June 2018.

\bibitem[Goodfellow et~al.(2015)Goodfellow, Shlens, and
  Szegedy]{goodfellow2015explaining}
I.~J. Goodfellow, J.~Shlens, and C.~Szegedy.
\newblock Explaining and harnessing adversarial examples, 2015.

\bibitem[Han et~al.(2022)Han, Liu, Liu, Jiang, Gu, Gao, and
  Chen]{han2022enhancing}
Y.~Han, J.~Liu, X.~Liu, X.~Jiang, L.~Gu, X.~Gao, and W.~Chen.
\newblock Enhancing adversarial transferability with partial blocks on vision
  transformer.
\newblock \emph{Neural Computing and Applications}, pages 1--14, 2022.

\bibitem[He et~al.(2015)He, Zhang, Ren, and Sun]{he2015deep}
K.~He, X.~Zhang, S.~Ren, and J.~Sun.
\newblock Deep residual learning for image recognition, 2015.

\bibitem[Li et~al.(2020)Li, Xu, Zhang, Yang, and Li]{li2020qeba}
H.~Li, X.~Xu, X.~Zhang, S.~Yang, and B.~Li.
\newblock Qeba: Query-efficient boundary-based blackbox attack, 2020.

\bibitem[Liu et~al.(2017)Liu, Chen, Liu, and Song]{liu2017delving}
Y.~Liu, X.~Chen, C.~Liu, and D.~Song.
\newblock Delving into transferable adversarial examples and black-box attacks,
  2017.

\bibitem[Madry et~al.(2019)Madry, Makelov, Schmidt, Tsipras, and
  Vladu]{madry2019deep}
A.~Madry, A.~Makelov, L.~Schmidt, D.~Tsipras, and A.~Vladu.
\newblock Towards deep learning models resistant to adversarial attacks, 2019.

\bibitem[Maho et~al.(2020)Maho, Furon, and Merrer]{maho2020surfree}
T.~Maho, T.~Furon, and E.~L. Merrer.
\newblock Surfree: a fast surrogate-free black-box attack, 2020.

\bibitem[Narodytska and Kasiviswanathan(2016)]{narodytska2016simple}
N.~Narodytska and S.~P. Kasiviswanathan.
\newblock Simple black-box adversarial perturbations for deep networks, 2016.

\bibitem[Papernot et~al.(2017)Papernot, McDaniel, Goodfellow, Jha, Celik, and
  Swami]{papernot2017practical}
N.~Papernot, P.~McDaniel, I.~Goodfellow, S.~Jha, Z.~B. Celik, and A.~Swami.
\newblock Practical black-box attacks against machine learning, 2017.

\bibitem[Pedregosa et~al.(2011)Pedregosa, Varoquaux, Gramfort, Michel, Thirion,
  Grisel, Blondel, Prettenhofer, Weiss, Dubourg, Vanderplas, Passos,
  Cournapeau, Brucher, Perrot, and Duchesnay]{scikit-learn}
F.~Pedregosa, G.~Varoquaux, A.~Gramfort, V.~Michel, B.~Thirion, O.~Grisel,
  M.~Blondel, P.~Prettenhofer, R.~Weiss, V.~Dubourg, J.~Vanderplas, A.~Passos,
  D.~Cournapeau, M.~Brucher, M.~Perrot, and E.~Duchesnay.
\newblock Scikit-learn: Machine learning in {P}ython.
\newblock \emph{Journal of Machine Learning Research}, 12:\penalty0 2825--2830,
  2011.

\bibitem[Rahmati et~al.(2020)Rahmati, Moosavi-Dezfooli, Frossard, and
  Dai]{rahmati2020geoda}
A.~Rahmati, S.-M. Moosavi-Dezfooli, P.~Frossard, and H.~Dai.
\newblock Geoda: a geometric framework for black-box adversarial attacks, 2020.

\bibitem[Rauber et~al.(2020)Rauber, Zimmermann, Bethge, and
  Brendel]{rauber2017foolboxnative}
J.~Rauber, R.~Zimmermann, M.~Bethge, and W.~Brendel.
\newblock Foolbox native: Fast adversarial attacks to benchmark the robustness
  of machine learning models in pytorch, tensorflow, and jax.
\newblock \emph{Journal of Open Source Software}, 5\penalty0 (53):\penalty0
  2607, 2020.
\newblock \doi{10.21105/joss.02607}.
\newblock URL \url{https://doi.org/10.21105/joss.02607}.

\bibitem[Rebuffi et~al.(2021)Rebuffi, Gowal, Calian, Stimberg, Wiles, and
  Mann]{rebuffi2021fixing}
S.-A. Rebuffi, S.~Gowal, D.~A. Calian, F.~Stimberg, O.~Wiles, and T.~Mann.
\newblock Fixing data augmentation to improve adversarial robustness, 2021.

\bibitem[Sooksatra and Rivas(2022)]{sooksatra2022evaluation}
K.~Sooksatra and P.~Rivas.
\newblock Evaluation of adversarial attacks sensitivity of classifiers with
  occluded input data.
\newblock \emph{Neural Computing and Applications}, pages 1--18, 2022.

\bibitem[Strisciuglio et~al.(2020)Strisciuglio, Lopez-Antequera, and
  Petkov]{strisciuglio2020enhanced}
N.~Strisciuglio, M.~Lopez-Antequera, and N.~Petkov.
\newblock Enhanced robustness of convolutional networks with a push--pull
  inhibition layer.
\newblock \emph{Neural Computing and Applications}, 32\penalty0 (24):\penalty0
  17957--17971, 2020.

\end{thebibliography}
\appendix

\section{Theorem proofs}
\label{appendix:A}
\begin{theorem}[Probability of updates]
\label{theo:A.1}
Let \(\Phi(x)\) be the cumulative distribution function of the standard normal distribution.
If \(g(x)\) is differentiable, the vector \(u \sim \mathcal{N}(0, \mathcal{I} p(t)^{2})\) and \(\|\nabla g(\theta_{t})\| \neq 0\), then
\begin{equation}
P[g(\theta_{t}) \leq g(\theta_{t} + u) \delta_{t}] = 1 - \Phi(\frac{g(\theta_{t})(1-\delta_{t}) + o(\|u\|)}{\|\nabla g(\theta_{t})\|p(t)})
\end{equation}
\end{theorem}
\begin{proof}
From first order Taylor expansion,
\[ g(\theta_{t}) \le g(\theta_{t}+u)\delta_{t} = (g(\theta_{t}) + {\nabla g(\theta_{t})}^{T} u + o(\|u\|))\delta_{t} \]
After reordering the terms, we obtain
\[{\nabla g(\theta_{t})}^{T} u \leq g(\theta_{t})(\delta_{t}-1) + o(\|u\|) \]
Since \(u \sim \mathcal{N}(0, \mathcal{I} p(t)^{2})\), we get
\[ {\nabla g(\theta_{t})}^{T} u = \sum_{l=1}^{d} {\nabla g(\theta_{t})}_{l} u_{l} \sim \mathcal{N}(0, \|\nabla g(\theta_{t})\|^{2}p(t)^{2})\]
Thus we will have,
\[ P[g(\theta_{t}) \leq g(\theta_{t} + u) \delta_{t}] = P[{\nabla g(\theta_{t})}^{T} u \leq g(\theta_{t})(\delta_{t}-1) + o(\|u\|)]  = 1 - \Phi(\frac{g(\theta_{t})(1-\delta_{t}) + o(\|u\|)}{\|\nabla g(\theta_{t})\|p(t)}) \]
\end{proof}

\begin{theorem}[Gradient Bound] 
\label{theo:A.2}
Let \(g(x)\) be differentiable, \(\delta_{t} = 1 - \|\nabla g(\theta_{t})\| \) and \(g(\theta_{t+1}) \leq \delta_{t} g(\theta_{t})\)
\begin{equation}
  \|\nabla g(\theta_{t})\| \leq o(\frac{1}{\cos{\angle \big(\nabla g(\theta_{t}), \theta_{t+1}-\theta_{t} \big)}} + \frac{\|\theta_{t+1}-\theta_{t}\|}{ g(\theta_{t})})
\end{equation}
\end{theorem}
\begin{proof}
Let \[\delta_{t} = 1 - \|\nabla g(\theta_{t})\| \]
From first order Taylor expansion, \[g(\theta_{t+1}) = g(\theta_{t}) + \|\nabla g(\theta_{t})\| \|\theta_{t+1}-\theta_{t}\| \cos{\angle \big(\nabla g(\theta_{t}), \theta_{t+1}-\theta_{t} \big)} + o(\|\theta_{t+1}-\theta_{t}\|)  \]
We can substitute \(g(\theta_{t+1})\) with \(g(\theta_{t+1}) \leq \delta_{t} g(\theta_{t})\), \[g(\theta_{t}) + \|\nabla g(\theta_{t})\| \|\theta_{t+1}-\theta_{t}\| \cos{\angle \big(\nabla g(\theta_{t}), \theta_{t+1}-\theta_{t} \big)} + o(\|\theta_{t+1}-\theta_{t}\|) \leq \delta_{t} g(\theta_{t})\]
We can substitute \(\delta_{t}\) and rearrange the terms and we obtain,
\[ \|\nabla g(\theta_{t})\| \leq o(\frac{1}{\cos{\angle \big(\nabla g(\theta_{t}), \theta_{t+1}-\theta_{t} \big)}} + \frac{\|\theta_{t+1}-\theta_{t}\|}{g(\theta_{t})})\]
\end{proof}

\section{Estimation of \(\frac{g(\theta_{t})}{g(\theta_{t-1})}\)}
\label{appendix:B}
For the figure below, we see that the ratio of \(\frac{g(\theta_{t})}{g(\theta_{t-1})}\) converges toward 1. Moreover, we can better see the trade-off between a smaller value of \(\delta_{t}\) (log 0.01 - const) with many, but small improvements, compared with a bigger \(\delta_{t}\) (log 0.1 - const) with fewer, but larger improvements.

\begin{tikzpicture}
\begin{axis}[
    enlargelimits=false,
    ylabel={\(\frac{g(\theta_{t})}{g(\theta_{t-1})}\)},
    ymin=0.90, ymax=1,
    ymajorgrids=true,
    no markers,
    height=7cm,
    style={thick},
    width=\linewidth,
    legend pos=south east,
]
\addplot [
    color=blue,
] table {l_avg_log0.01_sum.dat};
\addplot [
    color=brown,
] table {l_avg_log0.01_sqrt.dat};
\addplot [
    color=red,
] table {l_avg_log0.1_sum.dat};
\legend{log 0.01 - const, log 0.01 - sqrt, log 0.1 - const}
\end{axis}
\end{tikzpicture}

\section{Analysis of the modified binary search}
\label{appendix:C}
We considered the common binary search (\(\lfloor \log{(1-\delta_{t})^{-1}} + 1 \rfloor \) queries) and our improved binary search.
To prove the effectiveness of our technique, we run our method with different hyperparameters. We can see, from the figure below, that it permits to reduce the number of queries by a large factor.

\begin{tikzpicture}
\begin{axis}[
    enlargelimits=false,
    ylabel={\# of queries},
    ymin=1.5, ymax=11,
    ymajorgrids=true,
    no markers,
    style={thick},
    height=9cm,
    width=\linewidth,
    legend pos=north east,
]
\addplot [
    color=red,
] table {binreal_log0.01_sum.dat};
\addplot [
    color=red,
    dashed
] table {binexp_log0.01_sum.dat};
\addplot [
    color=brown,
] table {binreal_log0.01_sqrt.dat};
\addplot [
    color=brown,
    dashed
] table {binexp_log0.01_sqrt.dat};
\addplot [
    color=blue,
] table {binreal_log0.1_sum.dat};
\addplot [
    color=blue,
    dashed
] table {binexp_log0.1_sum.dat};

\legend{our method: log 0.01 - const, standard: log 0.01 - const, our method: log 0.01 - sqrt, standard: log 0.01 - sqrt, our method: log 0.1 - const, standard: log 0.1 - const}
\end{axis}
\end{tikzpicture}

\section{2D example of the \textsc{DeltaBound} attack}
\label{appendix:D}
To better understand the behavior of our attack, we plot some iterations on different functions and situations. The target point is always set at (0, 0) and we define all functions as a binary classifier based on the result of the respective conditions. We use the following functions to explore the behavior of our method across easy and complex functions.
\begin{itemize}
    \item \(f_1(x, y)  = x+y+0.1\)
    \item \(f_2(x, y)  = sin(100x + 100y + 1)\)
    \item \(f_3(x, y) = {sin(20x+0.5)}^2 + {cos(10y+0.1)}^2 - 0.7 \)
    \item \(f_4(x, y)  = \sum_{a=0}^{10} sin(15(a+5)x + 1) + sin(15(a+2)y + 4)\)
\end{itemize}

\begin{figure}[hh]
    \centering
    \begin{subfigure}{0.4\textwidth}
    \includegraphics[width=\textwidth]{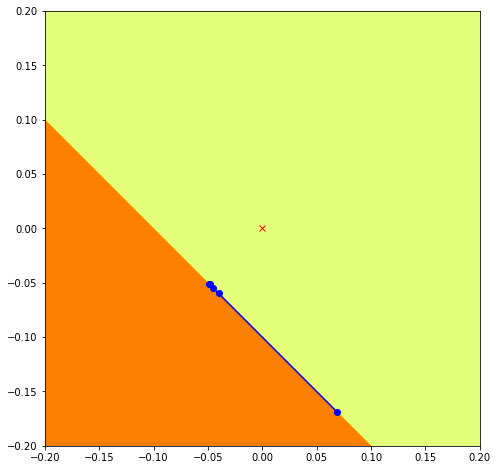}
    \caption{\(f_1(x, y) \leq 0\)}
    \end{subfigure}
    \begin{subfigure}{0.4\textwidth}
    \includegraphics[width=\textwidth]{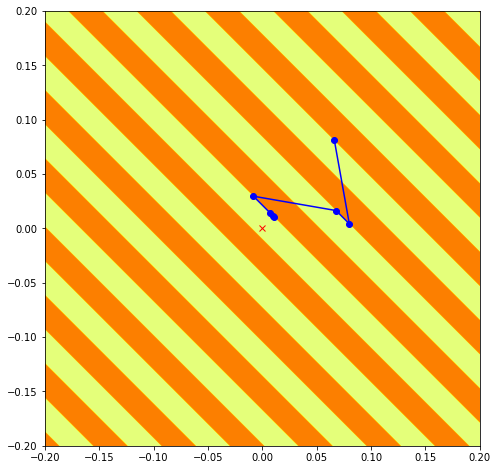}
    \caption{\(f_2(x, y) \leq 0\)}
    \end{subfigure}
\end{figure}
\begin{figure}[hh]
    \centering
    \begin{subfigure}{0.4\textwidth}
    \includegraphics[width=\textwidth]{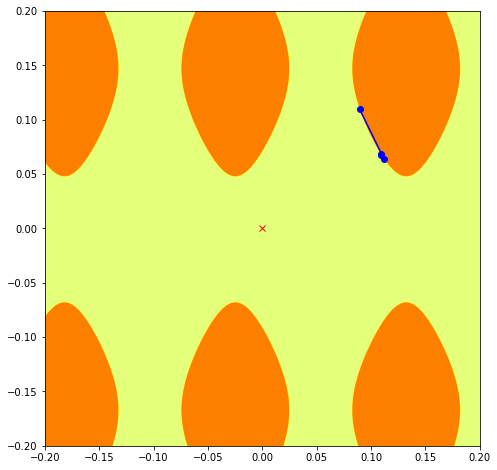}
    \caption{\(f_3(x, y) \leq 0\)}
    \end{subfigure}
    \begin{subfigure}{0.4\textwidth}
    \includegraphics[width=\textwidth]{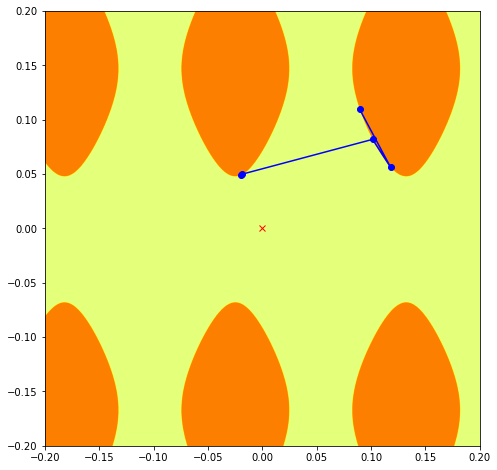}
    \caption{\(f_3(x, y) \leq 0\)}
   \end{subfigure}
\end{figure}
\begin{figure}[hh]
    \centering
    \begin{subfigure}{0.4\textwidth}
    \includegraphics[width=\textwidth]{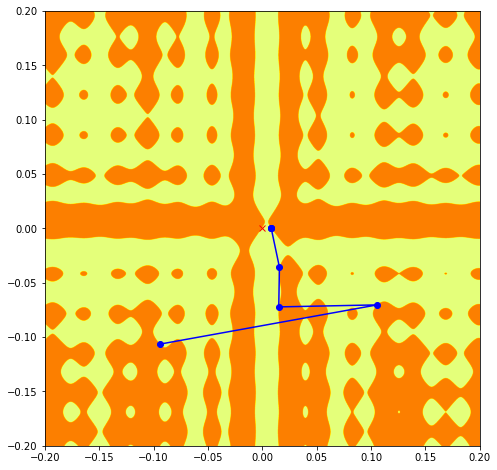}
    \caption{\( f_4(x, y) \leq 0\)}
    \end{subfigure}
    \begin{subfigure}{0.4\textwidth}
    \includegraphics[width=\textwidth]{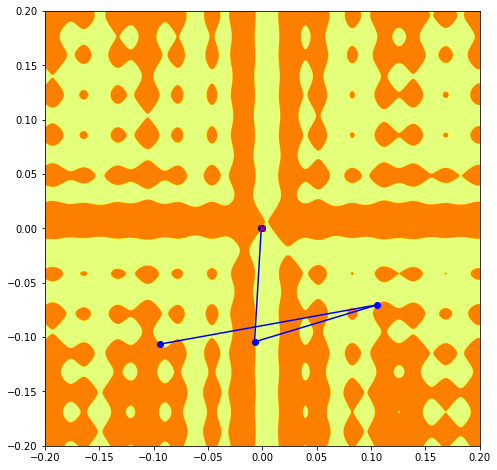}
    \caption{\(f_4(x, y) \leq 0\)}
    \end{subfigure}
\end{figure}

\end{document}